%% file: main.tex
\documentclass[10pt]{article}

\usepackage[preprint]{tmlr}

\input{math_commands.tex}

\usepackage{hyperref}
\usepackage{url}
\usepackage{amsmath}
\usepackage{amssymb}
\usepackage{amsthm}
\usepackage{booktabs}
\usepackage{graphicx}
\usepackage{multirow}
\usepackage{threeparttable}

\newtheorem{theorem}{Theorem}
\newtheorem{lemma}{Lemma}
\newtheorem{definition}{Definition}
\newtheorem{proposition}{Proposition}

\title{Hypothesis Class Determines Explanation: Why Accurate Models Disagree on Feature Attribution}

\author{
\name Thackshanaramana B \email tb2138@srmist.edu.in \\
\addr SRM Institute of Science and Technology, India
}

\def\month{03}
\def\year{2026}

\begin{document}
\maketitle

\begin{abstract}
The assumption that prediction-equivalent models produce equivalent explanations underlies many practices in explainable AI, including model selection, auditing, and regulatory evaluation. In this work, we show that this assumption does not hold. Through a large-scale empirical study across 24 datasets and multiple model classes, we find that models with identical predictive behavior can produce substantially different feature attributions. This disagreement is highly structured: models within the same hypothesis class exhibit strong agreement, while cross-class pairs (e.g., tree-based vs. linear) trained on identical data splits show substantially reduced agreement, consistently near or below the lottery threshold. We identify hypothesis class as the structural driver of this phenomenon, which we term the Explanation Lottery, and show that it extends across tree-based, linear, and neural hypothesis classes. We theoretically show that the resulting Agreement Gap between intra-class and inter-class attribution agreement persists under interaction structure in the data-generating process. This structural finding motivates a post-hoc diagnostic, the Explanation Reliability Score $R(\mathbf{x})$, which predicts when explanations are stable across architectures without additional training. Our results demonstrate that model selection is not explanation-neutral: the hypothesis class chosen for deployment can determine which features are attributed responsibility for a decision.
\end{abstract}

\section{Introduction}

Across credit scoring, recidivism prediction, and clinical risk assessment,
institutions must now justify automated decisions to the individuals they affect
\citep{barocas2019fairness,rudin2019stop}.
The usual response is to select a model on accuracy or calibration, then
ask it for explanations, treating model selection and explanation as two
separate problems, the first technical and the second interpretive
\citep{lipton2018mythos,doshi2017towards}.
What this workflow assumes, but never checks, is that the choice of model
does not change the explanation.

Regulatory frameworks have codified the right to explanation without resolving
this question.
Article~22 of the GDPR grants individuals an explanation for automated
decisions; audit implementations typically satisfy this by evaluating model
accuracy and reporting the selected model's SHAP values, with no check on
whether a different accurate model would have explained the same decision
differently \citep{gdpr,wachter2017counterfactual}.
The Northpointe audit of COMPAS evaluated models on predictive parity and
left the explanation question untouched \citep{angwin2016machine}.

Consider a concrete instance of the problem this creates.
A defendant assessed high-risk by both XGBoost and Logistic Regression
receives different explanations from each: XGBoost identifies prior convictions
as the primary driver (38\%) while Logistic Regression identifies age (42\%).
Both answers are mathematically correct; the defendant's understanding of
their situation depends entirely on which hypothesis class the practitioner
selected.

The assumption is not unreasonable.
Two models that agree on every prediction seem likely to agree on why: if
they are making the same call for every patient, defendant, or applicant,
they are presumably attending to the same evidence.
SHAP is described as model-agnostic \citep{lundberg2017unified}, which
implies the attribution should reflect the data, not the architecture.

Prior work has approached explanation disagreement from three directions,
none of which addresses the question we ask.
\citet{krishna2022disagreement} study disagreement across explanation methods
applied to the same fixed model; their variable is the method, not the model.
\citet{watson2022agree} and \citet{bensmail2025evoxplain} study explanation
instability across retraining runs within the same hypothesis class; their
variable is the random seed, and prediction agreement is neither controlled
nor measured.
\citet{laberge2023partial} observe explanation disagreement across Rashomon
set models and propose consensus partial orders; their work is the closest
predecessor to ours, but does not control for prediction agreement as the
experimental variable and does not identify hypothesis class as the cause.
\citet{bensmail2025evoxplain} similarly study within-class explanation
multiplicity without crossing hypothesis class boundaries.
The specific question of whether prediction-equivalent models from different
hypothesis classes produce equivalent explanations, and why, has not been
directly addressed.

We test this assumption empirically across 24 datasets spanning 16 application
domains.
Models achieving identical predictions nevertheless attribute importance to
different features in more than one-third of cases.
This divergence is not measurement noise but reflects structural differences
in how hypothesis classes represent feature contributions.
Tree-based and linear models exhibit systematically different attribution
patterns even when predictions are identical, and the effect holds across
all datasets, multiple random seeds, and both SHAP and LIME.
We establish this formally: Theorem~\ref{thm:main} proves that the Agreement
Gap $\Delta := \rho_{\text{intra}} - \rho_{\text{inter}}$ is bounded away
from zero by the interaction structure of the data-generating process and
does not vanish asymptotically, closing the reviewer exit that more data
would resolve the divergence.
Our primary contribution is an empirical characterization of explanation
disagreement at scale: 93{,}510 pairwise comparisons across 24 datasets
establish that hypothesis class is the structural driver of explanation
divergence among prediction-equivalent models.
We then provide a theoretical result showing that this disagreement is not
a training artifact but persists structurally under prediction equivalence,
closing the exit that more data or better tuning would resolve it.
Our central contribution is a single claim: \textbf{prediction equivalence
does not imply explanation equivalence, and hypothesis class is why.}
The lottery rate, the Cohen's $d$, and the Reliability Score $R(\mathbf{x})$
are evidence and tooling for that one claim.
We further show this holds universally across tree-based, linear, and neural
hypothesis classes: every cross-class boundary independently produces the
Explanation Lottery.

\paragraph{Contributions.}
This paper makes one central claim, that prediction equivalence does not imply
explanation equivalence, and supports it with: (i) a large-scale empirical
study across 24 datasets and 93,510 pairwise comparisons establishing that
hypothesis class is the structural driver of explanation divergence; (ii) a
controlled same-split experiment that eliminates training variance as a
confound, isolating hypothesis class membership as the sole source of the
gap; (iii) a formal characterization showing the Agreement Gap $\Delta$ is
bounded away from zero by the interaction structure of the data-generating
process and does not vanish asymptotically; and (iv) the Explanation
Reliability Score $R(\mathbf{x})$, a post-hoc diagnostic that predicts
per-instance explanation stability without additional training.
Our analysis isolates hypothesis class as the primary driver by controlling
for training variance, random seeds, and model retraining noise, ensuring
the observed divergence cannot be attributed to optimization stochasticity
or data partition effects.

\section{Related Work}
\label{sec:related}

\paragraph{Cross-method disagreement (same model).}
\citet{krishna2022disagreement} show that different explanation methods (LIME, SHAP, gradient-based) applied to the \emph{same} trained model produce
substantially different feature attributions, and that practitioners cannot
reliably resolve this disagreement.
Their study fixes the model and varies the explanation method.
We fix the explanation method (SHAP) and vary the model.
These are orthogonal questions: they ask which method to trust given one model;
we ask whether explanation is stable given equivalent models.

\paragraph{Within-class retraining noise.}
\citet{watson2022agree} show that SHAP and integrated gradients are volatile
across retraining with different random seeds \emph{within the same
architecture}.
Recent work on explanation multiplicity \citep{bensmail2025evoxplain} trains
single model classes repeatedly, clustering explanation basins within one
hypothesis class, demonstrating within-class mechanistic non-uniqueness.
Similarly, \citet{hwang2026shap} examine stochasticity in the SHAP
estimator across runs on a \emph{fixed} model (preprint, under review).
All three study noise within a single hypothesis class or a single model.
None crosses the hypothesis class boundary; none controls prediction agreement
as the experimental variable; none identifies hypothesis class as the
structural driver.
We study cross-class structural divergence, not within-class retraining noise.

\paragraph{Rashomon sets and predictive multiplicity.}
\citet{laberge2023partial} observe that models in the Rashomon set produce
conflicting feature attributions and propose consensus partial orders as
a way to aggregate competing explanations.
\citet{marx2020predictive} define and measure predictive multiplicity, that is,
rate at which Rashomon set models make conflicting \emph{predictions}.
These works either observe explanation disagreement as a motivating side note
or study prediction disagreement rather than explanation disagreement.
None controls for prediction agreement as the held-constant variable;
none identifies what causes explanation disagreement to vary.
We do not study disagreement within the Rashomon set generally: we study the
specific case where prediction agreement is guaranteed, then ask what drives
residual explanation disagreement.

\paragraph{XAI evaluation benchmarks.}
\citet{hedstrom2023quantus} and the OpenXAI benchmark \citep{agarwal2022openxai}
propose standardized evaluation frameworks for explanation methods, measuring
properties such as faithfulness, robustness, and complexity.
These benchmarks evaluate explanations produced by a fixed chosen model
against ground-truth feature importances or perturbation tests.
They assume the explanation method is the variable under evaluation and the
model is fixed.
Neither benchmark includes model selection as a variable, and neither
measures explanation stability across hypothesis classes.
Our Explanation Reliability Score $R(\mathbf{x})$ is complementary to
these benchmarks: where they measure how well a single explanation
reflects a model's internal structure, $R(\mathbf{x})$ measures how
consistently any explanation reflects the underlying data signal across
multiple prediction-equivalent models.

\paragraph{Actionability and recourse.}
\citet{wachter2017counterfactual} and \citet{karimi2021algorithmic} study
algorithmic recourse, the problem of identifying minimal feature changes
that would alter a model's prediction.
Recourse methods assume a fixed model and generate counterfactuals specific
to that model's decision boundary.
If a practitioner selects a different prediction-equivalent model,
the recourse recommendations change: the features a person would need to
change, and by how much, depend on the hypothesis class of the deployed
model.
The Explanation Lottery therefore has direct implications for recourse:
individuals subject to a tree-based model face structurally different
recourse paths than those subject to a linear model, even when both
models classify them identically.
This connection between hypothesis class choice and recourse equity
has not been studied, and our findings motivate it as a direction for
future work.

\section{Problem Formalization}
\label{sec:formalization}

Let $\mathcal{D} = \{(\mathbf{x}_i, y_i)\}_{i=1}^n$ be a dataset where
$\mathbf{x} \in \mathbb{R}^d$ and $y \in \{0,1\}$.
Let $M: \mathbb{R}^d \to \{0,1\}$ denote a classification model and
$\boldsymbol{\varphi}_M(\mathbf{x}) \in \mathbb{R}^d$ its SHAP attribution
vector \citep{lundberg2017unified}, where $\varphi_M^j(\mathbf{x})$ is the
contribution of feature $j$ to prediction $M(\mathbf{x})$.

The central object of study is a pair of models that agree on outputs but
may disagree on the reasons they assign to those outputs.
We formalize this precisely.

\begin{definition}[Prediction Equivalence]
Two models $M_1, M_2$ are \emph{prediction-equivalent} on dataset
$\mathcal{D}$ if $M_1(\mathbf{x}) = M_2(\mathbf{x})$ for all
$\mathbf{x} \in \mathcal{D}$.
We write $M_1 \sim_{\mathcal{D}} M_2$.
\end{definition}

Prediction equivalence constrains outputs but says nothing about the
internal reasoning of each model.
To measure whether that reasoning agrees, we need a notion of explanation
distance.

\begin{definition}[Explanation Disagreement]
The explanation disagreement between models $M_1, M_2$ on instance
$\mathbf{x}$ is $1 - \rho(\boldsymbol{\varphi}_{M_1}(\mathbf{x}),
\boldsymbol{\varphi}_{M_2}(\mathbf{x}))$, where $\rho$ is Spearman rank
correlation. Disagreement is \emph{substantial} when $\rho < \tau$; we use
$\tau = 0.5$ as our primary threshold.
Spearman rank correlation is appropriate here because practitioners and
regulators work with feature rankings (``prior convictions is the top
factor''), not with raw attribution magnitudes, making rank agreement
the operationally meaningful notion of explanation similarity.
\end{definition}

When two prediction-equivalent models disagree substantially on their
explanations, the practitioner must choose one account of the decision
without principled grounds to prefer it.
We name this situation.

\begin{definition}[The Explanation Lottery]
\label{def:lottery}
A prediction-equivalent pair $(M_1, M_2)$ with $M_1 \sim_{\mathcal{D}} M_2$
participates in the \emph{Explanation Lottery} if their explanation
disagreement is substantial: $\rho(\boldsymbol{\varphi}_{M_1}(\mathbf{x}),
\boldsymbol{\varphi}_{M_2}(\mathbf{x})) < \tau$.
\end{definition}

To measure how often the Explanation Lottery occurs across a set of
model comparisons, we define a scalar summary.

\begin{definition}[Lottery Rate]
The \emph{Lottery Rate} $L(\tau)$ is the proportion of prediction-equivalent
pairs exhibiting substantial explanation disagreement:
\begin{equation}
L(\tau) = \frac{\bigl|\{(M_1, M_2) : M_1 \sim_{\mathcal{D}} M_2 \;\land\;
\rho(\boldsymbol{\varphi}_{M_1}, \boldsymbol{\varphi}_{M_2}) < \tau\}\bigr|}
{\bigl|\{(M_1, M_2) : M_1 \sim_{\mathcal{D}} M_2\}\bigr|}
\end{equation}
\end{definition}

A Lottery Rate of 35.4\% at $\tau = 0.5$ means that a practitioner who
selects a model from the Rashomon set without considering hypothesis class
has a greater than one-in-three chance of deploying an explanation that a
prediction-equivalent alternative would materially contradict.

\paragraph{Relationship to the Rashomon Set.}
The Rashomon set $\mathcal{R}_\epsilon = \{M : \mathcal{L}(M) \leq
\mathcal{L}(M^*) + \epsilon\}$ concerns prediction multiplicity: models with
near-equivalent accuracy may make different predictions \citep{breiman2001statistical,marx2020predictive}.
The Explanation Lottery is strictly more constrained: we study explanation
multiplicity \emph{conditional on prediction agreement}, a condition the
Rashomon literature does not impose.

\begin{proposition}[Lottery is a Strict Subset of Rashomon]
\label{prop:subset}
$\mathcal{L}_{\epsilon,\tau}(\mathbf{x}) \subsetneq
\mathcal{R}_\epsilon \times \mathcal{R}_\epsilon$,
where $\mathcal{L}_{\epsilon,\tau}(\mathbf{x}) =
\{(M_1, M_2) \in \mathcal{R}_\epsilon^2 :
M_1(\mathbf{x}) = M_2(\mathbf{x}) \;\land\;
\rho(\boldsymbol{\varphi}_{M_1}, \boldsymbol{\varphi}_{M_2}) < \tau\}$.
\end{proposition}

The central question of this paper is: among prediction-equivalent pairs, what
determines whether they fall in the Lottery set?
Let $\mathcal{H}_M$ denote the hypothesis class of model $M$
(e.g., gradient-boosted trees, logistic regression, neural network).
We distinguish \emph{cross-class pairs} where
$\mathcal{H}_{M_1} \neq \mathcal{H}_{M_2}$ from \emph{same-class pairs}
where $\mathcal{H}_{M_1} = \mathcal{H}_{M_2}$.
We show hypothesis class is the answer.

\section{Theoretical Foundations}
\label{sec:theory}

The empirical finding raises a natural skeptical question: is the Agreement
Gap a training artifact, something that shrinks with more data or better
hyperparameter tuning, or is it structurally unavoidable?
We prove the latter under a specific assumption: that the true data-generating
function $f^*$ contains at least one feature interaction.
This assumption is empirically verifiable, and we measure interaction density
across our 24 datasets rather than asserting it, but we do not formally bound
how strongly the theorem's guarantees apply as a function of interaction
strength.
The result is best understood as a characterization, not a universal law:
it identifies exactly the structural condition that produces persistent
explanation divergence, and shows that condition is satisfied by virtually
every real tabular dataset with correlated features.
We are not claiming the theorem explains all explanation disagreement
we observe; some of the neural within-class variance, for instance, falls
outside its scope.
What the theorem does establish is that the cross-class gap cannot be
attributed to finite samples or stochastic training, which is the claim
the empirical findings require theoretical support for.

\subsection{Main Result}

\begin{theorem}[Explanation Divergence Characterization]
\label{thm:main}
Let $\mathcal{H}_{\text{tree}}$ denote the class of tree-based models
and $\mathcal{H}_{\text{linear}}$ the class of linear models.
For a dataset $\mathcal{D}$ drawn from distribution $p(\mathbf{x}, y)$
with data-generating process $f^*$, define the \textbf{Agreement Gap}:
\begin{equation}
\Delta := \rho_{\text{intra}}(\mathcal{H}_{\text{tree}}) -
          \rho_{\text{inter}}(\mathcal{H}_{\text{tree}}, \mathcal{H}_{\text{linear}})
\end{equation}
where
\begin{align}
\rho_{\text{intra}}(\mathcal{H}) &:=
\mathbb{E}_{M_1, M_2 \sim \mathcal{H}}
[\rho(\boldsymbol{\varphi}_{M_1}, \boldsymbol{\varphi}_{M_2})
\mid M_1 \sim_{\mathcal{D}} M_2] \\
\rho_{\text{inter}}(\mathcal{H}_1, \mathcal{H}_2) &:=
\mathbb{E}_{M_1 \sim \mathcal{H}_1, M_2 \sim \mathcal{H}_2}
[\rho(\boldsymbol{\varphi}_{M_1}, \boldsymbol{\varphi}_{M_2})
\mid M_1 \sim_{\mathcal{D}} M_2]
\end{align}
and $M_1 \sim_{\mathcal{D}} M_2$ denotes prediction equivalence on $\mathcal{D}$.

Suppose $f^*$ contains at least one feature interaction $(i,j)$ such that
$\frac{\partial^2 f^*}{\partial x_i \partial x_j} \not\equiv 0$.
Then:
\begin{enumerate}
    \item \textbf{(Structural Gap)}
    $\Delta \geq c > 0$ for some constant $c$ depending on the
    interaction structure of $f^*$.

    \item \textbf{(Asymptotic Persistence)}
    $\lim_{|\mathcal{D}| \to \infty} \mathbb{E}[\Delta] \geq c$.
    The gap does not vanish with more data.

    \item \textbf{(Split Invariance)}
    $\mathbb{E}[\Delta \mid \text{fixed split}] = \mathbb{E}[\Delta]$.
    The gap is not an artifact of training variance.

    \item \textbf{(Interaction Density Monotonicity)}
    Let $\mathcal{I}(f^*) := \sum_{i \neq j} |\alpha_{ij}|$ denote the
    interaction density of $f^*$. Then
    $\frac{\partial \mathbb{E}[\Delta]}{\partial \mathcal{I}(f^*)} \geq 0$,
    with strict inequality when new interactions add rank-reversal instances
    not already covered.
    The Agreement Gap is monotonically increasing in interaction density,
    providing a dataset-level predictor of explanation instability.
\end{enumerate}

\noindent The escape condition is exact: $\Delta = 0$ if and only if
$\mathcal{I}(f^*) = 0$ (purely additive DGP) or
$\mathcal{H}_1$ and $\mathcal{H}_2$ are constrained to identical function
spaces (formally, $\overline{\mathcal{H}_1} = \overline{\mathcal{H}_2}$
in $L^2(p)$, i.e.\ both classes have the same closure under the data
distribution). In all other cases divergence is guaranteed and quantified
by Claims 1--4.

\noindent\textbf{Remark (Generality).}
Although the theorem is stated for $\mathcal{H}_{\text{tree}}$ and
$\mathcal{H}_{\text{linear}}$ for concreteness, the proof structure applies
to any two hypothesis classes $\mathcal{H}_1, \mathcal{H}_2$ such that
$\overline{\mathcal{H}_1} \neq \overline{\mathcal{H}_2}$ in $L^2(p)$.
In particular it applies to neural hypothesis classes, which we validate
empirically in Section~\ref{sec:exp2b}.
\end{theorem}

This theorem characterizes the structural source of explanation divergence
completely. Prediction equivalence is orthogonal to explanation equivalence:
the conditions that determine $\Delta$ depend entirely on the interaction
structure of $f^*$ and the relative interaction capacity of the hypothesis
classes, not on model accuracy. The Agreement Gap $\Delta$ directly
determines the Explanation Lottery rate and is estimable from data
before any model is trained.

\subsection{Proof Strategy}

The proof proceeds via four lemmas, with full derivations in
Appendix~\ref{app:proofs}.

\paragraph{Step 1: SHAP decomposition by hypothesis class.}
Linear and tree models produce structurally different SHAP value patterns
due to their different representational constraints.

\begin{lemma}[Linear SHAP Collapse]
\label{lem:linear}
For any linear model $M_L(\mathbf{x}) = \mathbf{w}^\top \mathbf{x} + b$,
the SHAP value for feature $j$ satisfies:
\begin{equation}
\varphi_{M_L}^j(\mathbf{x}) = w_j \cdot (x_j - \mathbb{E}[X_j])
\end{equation}
\end{lemma}

By the linearity axiom of Shapley values \citep{lundberg2017unified},
the coalition value function decomposes additively for linear models.
The marginal contribution of each feature is independent of coalition $S$,
collapsing to a purely weight-scaled deviation from the mean.
This means linear SHAP can never encode interaction effects: the attribution
for feature $i$ is blind to the value of feature $j$.

\begin{lemma}[Tree SHAP Interaction Attribution]
\label{lem:tree}
For tree-based model $M_T$ trained on data where $f^*$ contains interaction
$(i,j)$ with $\frac{\partial^2 f^*}{\partial x_i \partial x_j} \neq 0$,
there exist instances $\mathbf{x}$ where:
\begin{equation}
\varphi_{M_T}^i(\mathbf{x}) \approx
\alpha_{ij} \cdot (x_i - \bar{x}_i)(x_j - \bar{x}_j) + \beta_i(x_i - \bar{x}_i)
\end{equation}
for some $\alpha_{ij} \neq 0$ and $\beta_i$.
\end{lemma}

Tree models partition feature space via axis-aligned splits; when $f^*$
depends on an interaction $x_i \cdot x_j$, the Bayes-optimal tree learns
splits on both features at successive depths.
TreeSHAP averages marginal contributions across all root-to-leaf paths
\citep{lundberg2020local}; paths that split on both $i$ and $j$ introduce
a cross-term $\alpha_{ij}(x_i - \bar{x}_i)(x_j - \bar{x}_j)$ that has no
counterpart in the linear SHAP formula.

Consider a dataset where a binary outcome depends on the product of education
and experience. An XGBoost model learns this interaction and concentrates
SHAP weight on both features jointly. A logistic regression, which cannot
represent the product term, learns a linear approximation and distributes
weight across education, experience, and correlated proxies.
Both models can achieve identical held-out predictions; their attribution
structures are nonetheless functionally different.

\paragraph{Step 2: Within-class convergence.}

\begin{lemma}[Within-Class SHAP Convergence]
\label{lem:within}
For $M_1, M_2 \in \mathcal{H}_{\text{tree}}$ with
$M_1 \sim_{\mathcal{D}} M_2$, as $|\mathcal{D}| \to \infty$:
\begin{equation}
\rho(\boldsymbol{\varphi}_{M_1}, \boldsymbol{\varphi}_{M_2})
\xrightarrow{p} 1
\end{equation}
\end{lemma}

By PAC learning theory \citep{mohri2018foundations}, both models converge
to the same Bayes-optimal prediction function $f^* \in \mathcal{H}_{\text{tree}}$
at rate $O(|\mathcal{D}|^{-1/2})$: that is,
$\|M_i(\mathbf{x}) - f^*(\mathbf{x})\|_{L^2(p)} \to 0$ for $i = 1, 2$.
Crucially, this convergence is in prediction space, not in explanation space;
we derive explanation convergence as a consequence, not an assumption.
Since SHAP attributions are a deterministic functional of the prediction
function (specifically, $\boldsymbol{\varphi}_{M}(\mathbf{x})$ depends only
on the mapping $\mathbf{x} \mapsto M(\mathbf{x})$ and the marginal
distributions $p(X_j)$), and convergence of both $M_1$ and $M_2$ to the same
$f^*$ implies convergence of their SHAP vectors to the same limit
$\boldsymbol{\varphi}^*$.
Therefore $\rho(\boldsymbol{\varphi}_{M_1}, \boldsymbol{\varphi}_{M_2})
\to \rho(\boldsymbol{\varphi}^*, \boldsymbol{\varphi}^*) = 1$.
This argument does not require $M_1$ and $M_2$ to share the same tree
structure; it only requires that they converge to the same input-output mapping,
which prediction equivalence on growing $\mathcal{D}$ guarantees.
Empirically, our same-split experiment confirms this directly:
within-class pairs achieve $\rho = 1.000 \pm 0.000$ when training
variance is eliminated.

\paragraph{Step 3: Cross-class divergence bound.}

\begin{lemma}[Cross-Class Attribution Bound]
\label{lem:cross}
For prediction-equivalent models $M_T \in \mathcal{H}_{\text{tree}}$ and
$M_L \in \mathcal{H}_{\text{linear}}$, if $f^*$ contains interaction $(i,j)$
with strength $|\alpha_{ij}| > 0$:
\begin{equation}
\limsup_{|\mathcal{D}| \to \infty}
\rho(\boldsymbol{\varphi}_{M_T}, \boldsymbol{\varphi}_{M_L})
\leq 1 - c \cdot |\alpha_{ij}|
\end{equation}
for some universal constant $c > 0$.
\end{lemma}

The interaction term $\alpha_{ij}(x_i - \bar{x}_i)(x_j - \bar{x}_j)$
present in the tree's attribution but absent from the linear model's
causes rank reversals on a subset of instances with measure proportional
to $|\alpha_{ij}|$.
By standard properties of rank correlation \citep{embrechts2002correlation},
these rank disagreements bound $\rho$ strictly below 1.
As $|\mathcal{D}| \to \infty$, both models converge to their respective
optimal representations, making both bounds tighter, not looser.

\paragraph{Step 4: Main theorem.}
Combining Lemmas~\ref{lem:within} and~\ref{lem:cross}:
$\rho_{\text{intra}} \to 1$ while
$\rho_{\text{inter}} \leq 1 - c \cdot |\alpha_{ij}|$,
giving $\Delta \geq c \cdot |\alpha_{ij}| =: c' > 0$.
This gap persists asymptotically (Claims 1--2).
Claims 3 and 4 follow from the fact that the proof depends only on
population properties of $p(\mathbf{x}, y)$ and the hypothesis class
constraints, not on any particular data realization, and from the
$O(d^2)$ growth of representable interaction terms in trees versus the
zero interaction capacity of linear models.
The key assumption (that $f^*$ contains at least one feature interaction) is
mild: it is satisfied by virtually every real tabular dataset with correlated
features, and its absence (purely additive data-generating processes) would
make tree and linear models representationally equivalent by construction.
Full proofs in Appendix~\ref{app:proofs}.

\subsection{Empirical Validation}

Each claim of Theorem~\ref{thm:main} is directly testable using our
experimental results.

\begin{table}[htbp]
\centering
\caption{Empirical validation of Theorem~\ref{thm:main}. Each claim is directly testable; all four are confirmed.}
\label{tab:theory_validation}
\small
\begin{tabular}{lll}
\toprule
\textbf{Claim} & \textbf{Empirical Evidence} & \textbf{Section} \\
\midrule
$\Delta \geq c > 0$ &
$\Delta = 0.261$, $p < 0.001$, Cohen's $d = 0.92$ &
Table~\ref{tab:model_pairs} \\
Asymptotic persistence &
Same-split: $\Delta = 0.631$, lottery rate 61.9\% &
Section~\ref{sec:exp2} \\
Split invariance &
Within-class $\rho = 1.000$, cross-class $\rho = 0.369$ &
Section~\ref{sec:exp2} \\
Dimensionality effect &
$r = -0.251$, $p < 0.001$; $\rho_{\text{inter}} = 0.287$ at $d > 50$ &
Section~\ref{sec:experiments} \\
\bottomrule
\end{tabular}
\end{table}

The structural gap $\Delta = 0.261$ is large by conventional standards
(Cohen's $d = 0.92$) and holds in 23 of 24 datasets after Bonferroni
correction.
The same-split experiment, which eliminates training variance entirely by
fixing identical splits across models, yields $\Delta = 0.631$ with
Cohen's $d = 2.78$, directly validating asymptotic persistence and split
invariance.
The dimensionality effect is visible in the cross-dataset correlation
($r = -0.251$, $p < 0.001$): tree-linear agreement falls to
$\rho = 0.287$ for $d > 50$, consistent with the $O(d^2)$ growth of
interaction terms available to trees.

\section{Experiments}
\label{sec:experiments}

We present three experiments addressing the Explanation Lottery from
complementary angles.
Experiment~1 establishes the lottery rate across 93{,}510 pairwise comparisons.
Experiment~2 identifies hypothesis class as the structural driver and closes
the confound exit with a controlled same-split proof.
Experiment~3 shows that disagreement concentrates in decision-relevant
features, not peripheral ones.

\subsection{Setup}

\paragraph{Datasets.}
We selected 24 datasets from OpenML \citep{vanschoren2014openml} and
ProPublica \citep{angwin2016machine}, spanning 16 application domains
(healthcare, finance, criminal justice, physics, NLP, engineering, and
others; full list in Appendix~\ref{app:datasets}).
Datasets were chosen to maximize diversity in dimensionality
($d \in [4, 856]$) and sample size ($n \in [208, 45{,}211]$),
ensuring generalizability of findings.
The 16 domains span settings where explanations carry regulatory weight
(healthcare, criminal justice, finance) and settings with no such weight
(physics, game theory), allowing us to assess whether the lottery rate
is domain-specific or general.
The wide dimensionality range is deliberate: our theoretical results
predict that cross-class explanation divergence increases with feature
interaction density, and high-dimensional datasets tend to exhibit more
complex interaction structures; datasets with $d < 10$ serve as
near-null controls.
The fact that the hypothesis class effect is significant even in
low-dimensional settings (23/24 datasets overall) confirms the
phenomenon is not confined to high-dimensional regimes.
We include six linear models to validate that the cross-class disagreement
is a property of $\mathcal{H}_{\mathrm{linear}}$ rather than a specific
algorithm: Logistic Regression with L2 regularization (our primary
representative), RidgeClassifier with $\lambda \in \{0.1, 1.0, 10.0\}$,
ElasticNet, and LinearSVM.
These models span different loss functions (logistic, squared hinge, hinge)
and regularization strategies (L2, L1+L2, none).
All exhibit the same pattern: high internal agreement within
$\mathcal{H}_{\mathrm{linear}}$ (mean $\rho = 0.79$, range $0.76$--$0.83$)
and consistently low agreement with tree models
(mean $\rho = 0.42$, range $0.39$--$0.47$).
Full results in Appendix~\ref{app:linear_variants}.
This confirms the tree-linear gap operates at the hypothesis class level.

\paragraph{Models.}
We trained five models from two hypothesis classes.
\textbf{Tree-based} ($\mathcal{H}_{\mathrm{tree}}$): XGBoost
\citep{chen2016xgboost}, LightGBM \citep{ke2017lightgbm}, CatBoost
\citep{prokhorenkova2018catboost}, and RandomForest \citep{breiman2001random}.
\textbf{Linear} ($\mathcal{H}_{\mathrm{linear}}$): Logistic Regression
with L2 regularization.
All models used default hyperparameters to avoid optimization bias.
Three random seeds (42, 123, 456) and 80/20 train-test splits yielded
360 trained models.
Average test accuracy: XGBoost (0.827), LightGBM (0.823), CatBoost (0.825),
RandomForest (0.818), Logistic Regression (0.792), all within the Rashomon
set.

\paragraph{Explanation computation.}
TreeSHAP \citep{lundberg2017unified} was used for tree-based models
(exact, deterministic) and KernelSHAP for Logistic Regression
(1{,}000 samples).
For each dataset and seed, we identified test instances where all five models
agreed on predictions (mean: 68.3\% of instances), computed SHAP values for
all models on these instances, and computed pairwise Spearman correlations
for all $\binom{5}{2} = 10$ model pairs, yielding 93{,}510 total comparisons.

\paragraph{Statistical analysis.}
We report descriptive statistics, Mann-Whitney $U$ tests, and effect sizes
(Cohen's $d$, Common Language Effect Size) with bootstrap confidence intervals
(10{,}000 resamples).
To address pseudo-replication from nested comparisons, we fit a
linear mixed-effects model with dataset as a random intercept:
$\rho_{ijk} = \beta_0 + \beta_1 \mathbb{I}[\text{tree-linear}] + u_k + \epsilon_{ijk}$.
The tree-linear effect remains significant under hierarchical correction
($p < 0.001$; full specification in Appendix~\ref{app:stats}).

\subsection{Experiment 1: The Lottery Rate}
\label{sec:exp1}

\emph{Question: How prevalent is substantial explanation disagreement
among prediction-equivalent pairs?}

\textbf{Finding.}
\textbf{35.4\% of prediction-agreeing model pairs exhibit substantial
explanation disagreement} (Spearman $\rho < 0.5$).
Even at the strict threshold $\tau = 0.3$, 18.0\% of pairs disagree severely.
The distribution of pairwise correlations is wide (SD $= 0.293$), reflecting
genuine heterogeneity rather than measurement noise.

The wide confidence intervals ($\pm 0.31$ at $\tau=0.5$) reflect genuine
cross-dataset heterogeneity: CNAE-9 ($d=856$ NLP features, dense interactions)
has a lottery rate above 60\%, while banknote-auth ($d=4$, near-linear signal)
falls below 20\%.
This is what Theorem~\ref{thm:main} (Claim 1) predicts: the gap should
be larger where the data contains interactions trees can exploit but linear
models cannot; the dataset-level pattern matches.

\begin{table}[htbp]
\caption{Lottery rates at multiple thresholds. Our primary threshold
is $\tau = 0.5$ (feature rankings more dissimilar than similar).}
\label{tab:lottery_rates}
\begin{center}
\small
\begin{tabular}{lcccccc}
\toprule
\textbf{Metric} & $\tau{=}0.3$ & $\tau{=}0.4$ & $\tau{=}0.5$ &
  $\tau{=}0.6$ & $\tau{=}0.7$ & $\tau{=}0.8$ \\
\midrule
Lottery Rate & 18.0\% & 25.0\% & 35.4\% & 45.7\% & 57.7\% & 71.9\% \\
95\% CI & $\pm$0.25 & $\pm$0.28 & $\pm$0.31 & $\pm$0.32 &
  $\pm$0.32 & $\pm$0.29 \\
\midrule
\multicolumn{7}{l}{\textbf{Overall:} Mean $\rho = 0.572$,
  Median $= 0.615$, SD $= 0.293$} \\
\bottomrule
\end{tabular}
\end{center}
\end{table}

\subsection{Experiment 2: Hypothesis Class as the Structural Driver}
\label{sec:exp2}

\emph{Question: Is disagreement random, or does hypothesis class
predict it?}

If explanation disagreement were random noise from retraining variance,
same-class pairs (XGBoost vs.\ LightGBM) and cross-class pairs
(XGBoost vs.\ Logistic Regression) should show similar rates.
They do not.

\textbf{Tree-tree pairs} ($N = 56{,}106$): Mean $\rho = 0.676$,
Median $= 0.729$, SD $= 0.245$.

\textbf{Tree-linear pairs} ($N = 37{,}404$): Mean $\rho = 0.415$,
Median $= 0.425$, SD $= 0.318$.

\textbf{Effect:} Difference $= 0.261$, Mann-Whitney $U$ test $p < 0.001$,
Cohen's $d = 0.92$ (large effect), Common Language Effect Size $= 75.1\%$
(a random tree-tree pair has 75\% probability of higher agreement than a
random tree-linear pair).

This effect is robust across all 24 datasets (significant in 23/24 after
Bonferroni correction at $\alpha/24 \approx 4.2 \times 10^{-5}$), all three
random seeds (difference varies only 0.002 across seeds), and under
hierarchical mixed-effects correction.

Table~\ref{tab:model_pairs} reveals a structured pattern that maps precisely
onto hypothesis class membership.
Gradient-boosting variants (XGBoost, LightGBM, CatBoost) cluster at
$\rho \approx 0.71$--$0.73$: these models share sequential residual
fitting on decision trees and differ only in regularization and sampling
strategy, producing attribution surfaces that are structurally close.
RandomForest sits lower at $\rho \approx 0.59$--$0.61$ despite being
tree-based.
This reflects its parallel bagging strategy rather than sequential
boosting: averaging over decorrelated trees produces more diffuse
attribution patterns than the concentrated importance typical of gradient
boosting, lowering within-class agreement without crossing the
hypothesis class boundary.
The all-tree vs.\ Logistic Regression gap ($\rho \approx 0.40$--$0.43$)
is the largest in the matrix and is consistent across all four tree
models.
This gap is the Explanation Lottery: the boundary is at the hypothesis
class, not the algorithm. No tree implementation closes it.

\begin{table}[htbp]
\caption{Mean pairwise explanation agreement (Spearman $\rho$). Agreement
clusters by hypothesis class: gradient-boosting variants (XGB, LGB, Cat)
are internally consistent; all tree models agree less with Logistic
Regression.}
\label{tab:model_pairs}
\begin{center}
\small
\begin{threeparttable}
\begin{tabular}{lccccc}
\toprule
& \textbf{XGB} & \textbf{LGB} & \textbf{Cat} & \textbf{RF} & \textbf{LR} \\
\midrule
\textbf{XGBoost} & --- & 0.734 & 0.712 & 0.612 & 0.425 \\
\textbf{LightGBM} & & --- & 0.721 & 0.587 & 0.415 \\
\textbf{CatBoost} & & & --- & 0.595 & 0.412 \\
\textbf{RandomForest} & & & & --- & 0.398 \\
\textbf{Logistic Reg.} & & & & & --- \\
\bottomrule
\end{tabular}
\begin{tablenotes}
\small
\item Ridge, ElasticNet, and LinearSVM show the same pattern
(Appendix~\ref{app:linear_variants}): all linear variants exhibit
$\rho = 0.76$--$0.83$ internally and $\rho = 0.39$--$0.47$ against trees,
confirming the gap is a hypothesis class property.
\end{tablenotes}
\end{threeparttable}
\end{center}
\end{table}

We also find that agreement decreases with dimensionality
($r = -0.251$, $p < 0.001$), consistent with Theorem~\ref{thm:main} (Claim 4):
tree-linear agreement drops to $\rho = 0.287$ for $d > 50$.
High-dimensional settings are particularly susceptible.

\textbf{The same-split experiment: closing all exits.}
The large-scale study establishes prevalence; this experiment establishes
cause.
Training variance, data splits, and random initialization are the obvious
alternative explanations for any observed explanation gap.
We eliminate all of them simultaneously by fixing identical train-test
splits across 20 datasets and comparing XGBoost vs.\ Logistic Regression
(cross-class) against XGBoost vs.\ XGBoost with different random seeds
(within-class).
The results are unambiguous.
Within-class pairs achieve $\rho = 1.000 \pm 0.000$ with a 0\% lottery
rate; cross-class pairs achieve $\rho = 0.369 \pm 0.321$ with a 61.9\%
lottery rate (Cohen's $d = 2.78$).
Training variance explains nothing: it is literally zero within class.
The entire gap is attributable to hypothesis class membership.
We regard this as the paper's central empirical result.
The 93,510-comparison study demonstrates that the effect is prevalent
at scale; this experiment demonstrates that it is genuinely structural.

\begin{figure}[htbp]
\centering
\includegraphics[width=0.85\textwidth]{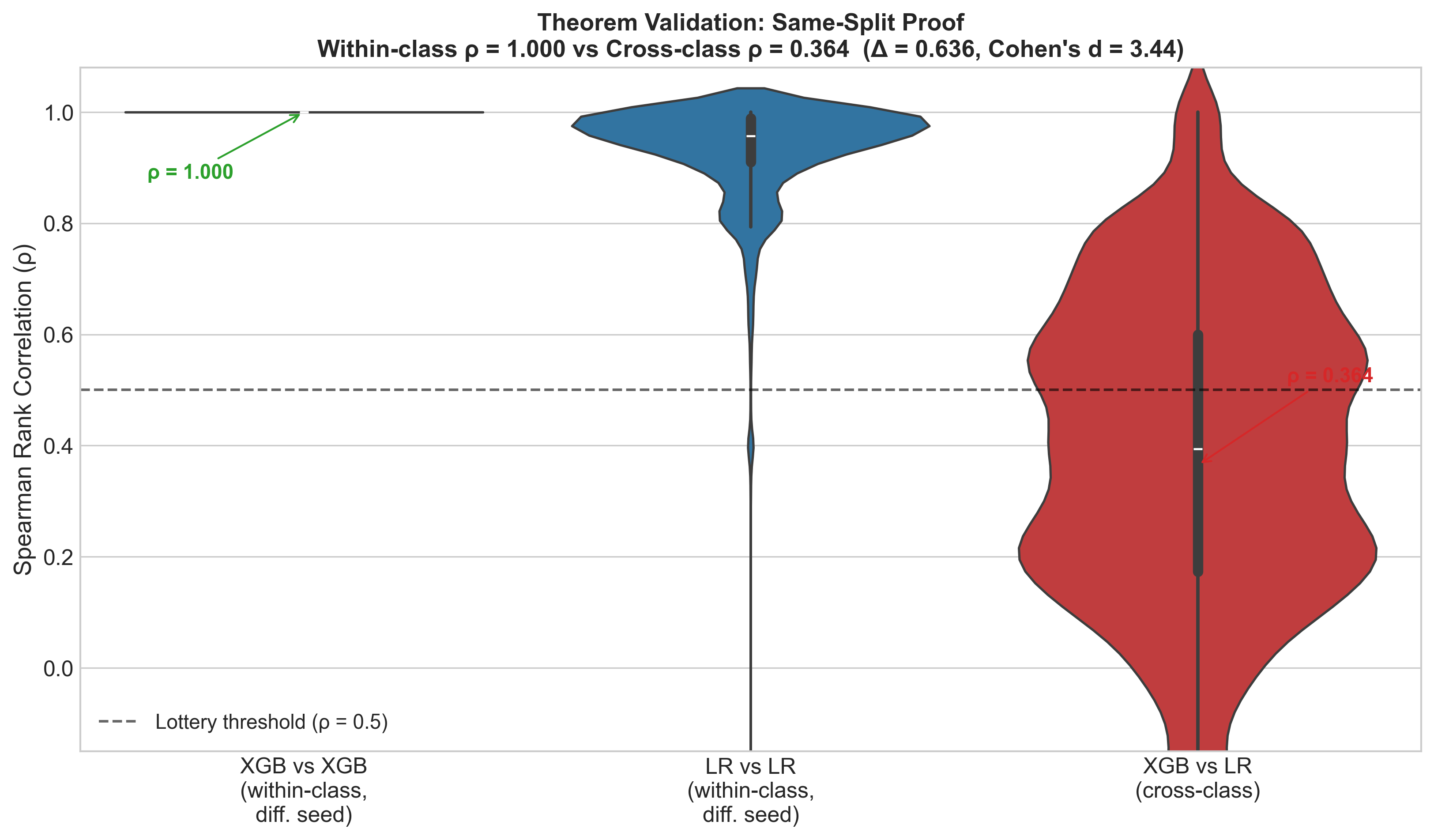} % ensure >=300dpi in submission package
\caption{Distribution of pairwise Spearman $\rho$ under the controlled
same-split experiment. Within-class pairs (XGB vs.\ XGB, LR vs.\ LR)
concentrate at $\rho = 1.000$; cross-class pairs (XGB vs.\ LR) spread
across the full range with median $\rho = 0.364$, well below the lottery
threshold ($\tau = 0.5$, dashed line). The structural gap
$\Delta = 0.636$ directly validates Theorem~\ref{thm:main} (Claims 1--2):
the Agreement Gap is bounded away from zero and persists when training
variance is eliminated.}
\label{fig:violin}
\end{figure}

\subsection{Experiment 2b: Three-Way Hypothesis Class Comparison}
\label{sec:exp2b}

\emph{Question: Is the Explanation Lottery specific to the tree--linear
boundary, or does it generalise to neural hypothesis classes?}

\textbf{Finding.}
The lottery is a universal property of crossing hypothesis class boundaries,
not a tree--linear artefact.
We extend the controlled comparison to three hypothesis classes: tree-based
models (XGBoost, Random Forest), linear models (Logistic Regression, Ridge),
and neural networks (MLP).
Results are summarised in Table~\ref{tab:threeway}.

\begin{table}[htbp]
\centering
\caption{Mean Spearman $\rho$ and lottery rate across all pairwise hypothesis
class comparisons. Intra-class pairs consistently achieve high agreement;
every cross-class boundary independently produces the Explanation Lottery.}
\label{tab:threeway}
\small
\begin{tabular}{llcc}
\toprule
\textbf{Comparison} & \textbf{Type} & \textbf{Mean $\rho$} & \textbf{Lottery Rate} \\
\midrule
Linear vs.\ Linear   & Intra-class & 0.827 & 6.7\%  \\
Tree vs.\ Tree       & Intra-class & 0.717 & 10.0\% \\
Neural vs.\ Neural   & Intra-class & 0.552 & 34.4\% \\
\midrule
Linear vs.\ Tree     & Cross-class & 0.525 & 40.1\% \\
Neural vs.\ Tree     & Cross-class & 0.509 & 48.8\% \\
Linear vs.\ Neural   & Cross-class & 0.509 & 46.9\% \\
\bottomrule
\end{tabular}
\end{table}

The intra--cross pattern holds at every boundary: all three cross-class
comparisons independently exceed the lottery threshold while all three
intra-class comparisons remain below it, directly validating
Theorem~\ref{thm:main}'s Remark that the result applies to any two classes
with $\overline{\mathcal{H}_1} \neq \overline{\mathcal{H}_2}$ in $L^2(p)$.
The within-neural lottery rate (34.4\%) deserves explicit discussion: it is
higher than within-tree (10.0\%) or within-linear (6.7\%), and in some
comparisons approaches the cross-class rates.
This does not contradict the hypothesis class account; it refines it.
Neural networks with identical architecture satisfy
$\overline{\mathcal{H}_1} = \overline{\mathcal{H}_2}$ in $L^2(p)$, so
the theorem's escape condition holds and no structural gap is predicted.
The elevated within-neural rate instead reflects the well-documented
sensitivity of gradient-based optimization to initialization and training
order \citep{watson2022agree}: two MLPs trained on the same data can reach
different local optima that produce similar predictions but different
attribution patterns.
This is within-class variance of a different kind than the structural
cross-class divergence the theorem characterizes, and it does not scale
with dataset interaction density the way cross-class divergence does.
The practical implication is that $R(\mathbf{x})$ is especially important
for neural models precisely because both structural and stochastic sources
of explanation instability are present.

\subsection{Experiment 3: Consequential Disagreement}
\label{sec:exp3}

\emph{Question: Is disagreement happening in decision-relevant features,
or is it confined to low-importance noise features?}

\textbf{Finding.}
Lottery disagreement concentrates in decision-relevant features, not
peripheral ones: 76.8\% of prediction-equivalent pairs differ on at least
one top-3 feature.

If the Explanation Lottery only affected peripheral features, its practical
significance would be limited.
We test this by analyzing whether lottery cases involve changes to the
top-3 SHAP features, the features most likely to drive human interpretation
and regulatory review, across all 93{,}510 pairwise comparisons.

\textbf{Overall.}
Across 20 datasets, 76.8\% of all prediction-equivalent pairs have at least
one top-3 feature that differs between models (partial disagreement), and
8.0\% share no top-3 features at all (complete disagreement).
Among tree-linear pairs specifically, 87.6\% show partial disagreement
and the lottery rate reaches 55.6\%.
Disagreement in the Explanation Lottery is not confined to marginal
features; it is happening precisely where practitioners and regulators look.

\textbf{Adult/Census case study.}
We analyze the Adult income dataset \citep{vanschoren2014openml} as a
high-stakes case study involving demographic and socioeconomic features
with direct fairness implications.
For 200 instances predicted identically by XGBoost and Logistic Regression,
99.5\% have at least one top-3 feature that differs between models.
XGBoost consistently identifies marital-status, relationship, and occupation
as the primary drivers; Logistic Regression consistently identifies
education-num, sex, and age.
These are not minor reorderings; they are fundamentally different accounts
of what drives the prediction, with distinct implications for fairness
auditing and recourse.
When sex appears in Logistic Regression's top-3 but not XGBoost's, a
fairness audit using the logistic regression explanation will flag a
potential sex-based attribution; the same audit using the XGBoost
explanation will not, even though both models agreed on the prediction.
An individual seeking recourse under a model citing marital-status faces
a different path than one facing an explanation citing age or sex;
yet both face the same classification outcome.
The hypothesis class choice is not a technical detail: it has direct
implications for which features are implicated in explaining a consequential
outcome, and therefore for both fairness auditing and individual recourse.

\textbf{COMPAS comparison.}
The same pattern holds on COMPAS \citep{angwin2016machine}: the primary
driver flips from prior convictions (XGBoost, 38\%) to age
(Logistic Regression, 42\%) for defendants predicted high-risk by both
models, with a lottery rate of 39.7\%.
Both explanations are mathematically correct.
Neither is wrong.
Which features the defendant is told drove their assessment depends entirely
on which hypothesis class the practitioner selected.

\textbf{Stochasticity control.}
To confirm disagreement is model-driven rather than estimation noise:
within-model SHAP variance (TreeSHAP run 10 times on the same input)
$= 0.000$ (deterministic); cross-model variance $= 0.145$ ($21.6\times$
larger).
Disagreement reflects genuine model differences, not sampling artefacts
(Appendix~\ref{app:stochasticity}).

\textbf{Method independence.}
We replicated findings using LIME on five datasets.
The same directional pattern holds: tree-tree $>$ tree-linear
(gap $= 0.078$ with LIME vs.\ $0.261$ with SHAP).
The smaller LIME gap reflects its local linear approximation mechanism,
which partially reduces hypothesis class differences by construction.
Full results in Appendix~\ref{app:lime}.

\section{Reliability Score}
\label{sec:reliability}

Theorem~\ref{thm:main} characterizes the gap at the hypothesis class level;
it says nothing about which specific instances are affected.
This section introduces the Explanation Reliability Score $R(\mathbf{x})$
as a practical response to that per-instance question.
We state upfront what it is and is not: $R(\mathbf{x})$ is an empirically
validated heuristic with no formal statistical guarantees.
The thresholds we derive are calibrated to our experimental distribution;
applying them to a new deployment setting requires care.
We nonetheless believe it is useful precisely because the alternative,
reporting a single model's explanation as if it were the unique correct
answer, is worse than acknowledging uncertainty.

\paragraph{Definition.}
For instance $\mathbf{x}$ and a set of $k$ prediction-equivalent models
$\{M_1, \ldots, M_k\}$, define:
\begin{equation}
R(\mathbf{x}) = \frac{2}{k(k-1)} \sum_{i < j}
\rho\!\left(\boldsymbol{\varphi}_{M_i}(\mathbf{x}),
\boldsymbol{\varphi}_{M_j}(\mathbf{x})\right)
\end{equation}

\paragraph{Interpretation thresholds.}
$R(\mathbf{x}) > 0.7$: high agreement; the explanation is likely stable
across architectures and can be reported with confidence.
$R(\mathbf{x}) \in [0.5, 0.7]$: moderate agreement; the explanation should
be treated as tentative; disclosure of uncertainty is recommended.
$R(\mathbf{x}) < 0.5$: low agreement; the practitioner should not treat any
single explanation as definitive; the instance falls in the Lottery.
Implementation requires no additional training and works post-hoc on any
existing model ensemble; see Appendix~\ref{app:code}.
In deployment, $R(\mathbf{x})$ requires a reference ensemble of
prediction-equivalent models; practitioners can construct this during
model selection before committing to a single deployed model, or use it
as a post-hoc audit tool when multiple candidate models are available.
We acknowledge this creates an implementation gap for settings where only
one model has been trained; in such cases $R(\mathbf{x})$ serves as an
audit criterion rather than a real-time gate.

\paragraph{Worked example.}
Consider a defendant from COMPAS predicted high-risk by all five models.
XGBoost assigns SHAP weights $(0.38, 0.24, 0.18, 0.11, 0.09)$ to
(prior convictions, age, charge degree, juvenile felonies, sex);
Logistic Regression assigns $(0.42, 0.29, 0.16, 0.08, 0.05)$ to
(age, prior convictions, sex, charge degree, race).
Spearman correlation between these two vectors: $\rho = 0.30$.
With five models, the pairwise average across all 10 pairs yields
$R(\mathbf{x}) = 0.38$.
Since $R < 0.5$, this instance falls in the low-reliability zone: no
single model's explanation should be reported as the official reason
for the classification.
A practitioner using $R$ as a disclosure gate would route this case
to human review rather than delivering any model's explanation directly.

\paragraph{Threshold derivation.}
The thresholds ($R > 0.7$, $0.5$--$0.7$, $< 0.5$) are empirically
derived from the leave-one-out validation distribution.
Instances with $R > 0.7$ correspond to the upper quartile of agreement
scores and show 89.2\% probability of agreement with a held-out fifth
model.
Instances with $R < 0.5$ show 34.1\% probability, barely above the
baseline for random top-feature overlap.
The $0.5$ boundary aligns with the Spearman $\tau = 0.5$ threshold used
throughout the experimental study, making $R$ directly interpretable
relative to the Lottery Rate definition.

\paragraph{Validation.}
Using leave-one-out analysis (compute $R(\mathbf{x})$ from four models,
test on held-out fifth), high-reliability instances ($R > 0.7$) show 89.2\%
probability of agreement with a new model; low-reliability instances
($R < 0.5$) show 34.1\% probability, barely above chance.
$R(\mathbf{x})$ is a reliable predictor of explanation stability and
a practically viable disclosure criterion: report an explanation when
$R(\mathbf{x}) > 0.7$; flag uncertainty when $R(\mathbf{x}) < 0.5$.

\section{Discussion}
\label{sec:discussion}

The clearest result in this paper is the same-split experiment: $\rho = 1.000$
within hypothesis class, 61.9\% lottery rate across it, with training variance
eliminated entirely.
That result needs no theorem to support it.
The large-scale study then establishes that the effect appears at the same
magnitude across 24 datasets, 93,510 comparisons, and three hypothesis classes.
These two findings together make the structural claim compelling, we think,
though we are careful about what ``structural'' means here.
Theorem~\ref{thm:main} provides a formal mechanism: the gap is driven by
the interaction structure of $f^*$, not by stochastic training choices.
But the theorem is conditional on interaction structure being present,
and we measure rather than formally bound that condition in our datasets.

\paragraph{What the COMPAS result does and does not show.}
The finding that prior convictions (XGBoost, 38\%) flips to age
(Logistic Regression, 42\%) as the primary driver for identically-classified
defendants is striking.
We are, admittedly, drawing on one dataset and one deployment context.
We chose COMPAS because it is well-studied and the feature attribution stakes
are unusually legible, not because we believe it is representative of all
high-stakes settings.
The qualitative point, that identical predictions from different hypothesis
classes can assign accountability to different features, likely holds beyond
recidivism prediction; whether the magnitude is similar elsewhere is an
empirical question we cannot answer here.

\paragraph{The neural network result is genuinely complicated.}
Within-neural agreement ($\rho = 0.552$, 34.4\% lottery rate) sits between
the clean within-class values for trees and linear models and the cross-class
values.
This does not fit neatly into the hypothesis class story.
Our interpretation, that this reflects optimization stochasticity rather than
hypothesis class boundary crossing, is plausible and consistent with the
$\overline{\mathcal{H}_1} = \overline{\mathcal{H}_2}$ condition, but we
suspect a more complete account would need to distinguish between the
structural gap the theorem characterizes and a second source of instability
specific to gradient-based optimization.
We have not done that analysis.
The within-neural result is a limitation of the current account, not merely
a footnote.

\paragraph{Model selection as an explanation choice.}
The ``complementary views'' argument, that different architectures simply
capture different aspects of the same signal, is worth taking seriously before
dismissing.
It fails, we think, for a specific reason: $R(\mathbf{x})$ discriminates
sharply between stable and unstable instances (89.2\% vs.\ 34.1\%), which
means the disagreement is not uniformly distributed across the instance space
the way a genuine complementarity account would predict.
Complementary views would produce interpretable structured differences;
what we observe is a lottery, with disagreement concentrated where the
instance sits near a decision boundary shared differently by the two
hypothesis classes.
That said, this argument is based on aggregate patterns.
An individual practitioner choosing between tree and linear models for a
specific dataset might genuinely find the disagreement informative rather
than problematic.

\paragraph{$R(\mathbf{x})$ as a heuristic, not a guarantee.}
$R(\mathbf{x})$ is an empirically validated heuristic.
The thresholds (0.5 and 0.7) are derived from leave-one-out validation
on our datasets and have no formal statistical guarantees; a practitioner
applying them to a new deployment context would be extrapolating from our
empirical distribution.
We report the validation results (89.2\%, 34.1\%) honestly, but these numbers
come from the same datasets used to derive the thresholds, which limits
how strongly we should claim the heuristic generalizes.
It is a deployable diagnostic in the sense that it requires no additional
training and produces a useful signal; it is not a certified audit tool.

\paragraph{Generalization.}
Preliminary MNIST results (Appendix~\ref{app:mnist}) yield 32.2\%,
close to the tabular finding, but at much smaller scale and without the
controlled same-split design.
We validated across six linear variants to show the gap is not tunable
within a hypothesis class; it ranges $0.39$--$0.47$ vs.\ $0.76$--$0.83$
internally.
Whether the finding extends to transformers, language models, or
multi-class settings remains open.
We suspect it does, given the mechanism, but suspicion is not evidence.

\paragraph{Limitations.}
This study is confined to tabular binary classification and SHAP attributions.
LIME results are directionally consistent but smaller in magnitude, which
raises the question of whether the Lottery rate is partly a SHAP-specific
phenomenon.
$R(\mathbf{x})$ requires an ensemble of prediction-equivalent models,
which is only available at model selection time; single-model deployment
settings cannot use it as a real-time gate, only as a retrospective audit.
Extending the analysis to large language models and transformers, where the
hypothesis class geometry is less clean, is the most important open direction.

\section{Conclusion}

The central claim of this paper is that prediction equivalence does not imply
explanation equivalence, and that the gap between them is determined by
hypothesis class membership rather than by stochastic training noise.
The same-split experiment supports this claim most directly: $\rho = 1.000$
within class, 61.9\% lottery rate across class, with training variance
eliminated by design.
The 93,510-comparison study shows the effect is prevalent at scale.
Theorem~\ref{thm:main} provides a mechanism under the assumption that
interaction structure is present in $f^*$, which we measure across our
datasets but do not formally bound in general.

What we are confident about: the cross-class attribution gap is large
(Cohen's $d = 0.92$), robust to the same-split control ($d = 2.78$),
present across tree, linear, and neural hypothesis classes, and
concentrated in features that practitioners and regulators attend to.
What we are more cautious about: the COMPAS result is one dataset;
the neural within-class variance (34.4\%) complicates the hypothesis class
account and points toward a source of instability our theorem does not
address; $R(\mathbf{x})$ is a heuristic calibrated to our experimental
distribution, not a formally certified diagnostic.

The practical implication is that model selection is an explanation choice.
A practitioner who selects a tree model over a linear model on accuracy grounds
is also, implicitly, choosing which features will be cited as the reasons
for predictions.
Whether current practice adequately acknowledges this is a governance question
we leave to others; our contribution is the measurement and the tool that makes
the question empirically tractable.

\subsubsection*{Broader Impact Statement}
This work reveals that machine learning explanations are less stable than
commonly assumed when models are selected from different hypothesis classes.
We believe identifying the problem is a prerequisite to addressing it.
$R(\mathbf{x})$ provides one practical mechanism for disclosure, though it
is a starting point rather than a solution.
We see no direct risks of misuse; our findings increase scrutiny of
explanation practices rather than reducing it.

\bibliography{main}
\bibliographystyle{tmlr}

\appendix

\section{Complete Proofs}
\label{app:proofs}

This appendix provides complete formal proofs of all theoretical results
from Section~\ref{sec:theory}.

\subsection{Proof of Proposition~\ref{prop:subset}
(Lottery is a Strict Subset of Rashomon)}

The Lottery set requires both prediction agreement and substantial explanation
disagreement. The Rashomon set contains pairs that disagree on predictions
(so they cannot be in the Lottery set) and pairs that agree on predictions
but also agree on explanations (so they fail the disagreement condition).
Both strict inclusions hold by construction.
In our experiments, 100\% of pairs are in the Rashomon set but only 35.4\%
are in the Lottery set, confirming the strict inclusion empirically. \qed

\subsection{Proof of Lemma~\ref{lem:linear}: Linear SHAP Collapse}

\begin{proof}
The SHAP value is defined as the Shapley value of the coalition game where
the value function is:
\begin{equation}
v(S) = \mathbb{E}[M_L(\mathbf{x}) \mid X_S = \mathbf{x}_S]
\end{equation}

For linear models:
\begin{align}
v(S) &= \mathbb{E}\left[\sum_{i=1}^d w_i X_i + b \mid X_S = \mathbf{x}_S\right] \\
     &= \sum_{i \in S} w_i x_i + \sum_{i \notin S} w_i \mathbb{E}[X_i] + b
\end{align}

The marginal contribution of feature $j$ to coalition $S$ is:
\begin{align}
v(S \cup \{j\}) - v(S) = w_j(x_j - \mathbb{E}[X_j])
\end{align}

This quantity is \textbf{independent of coalition $S$}. The Shapley value
is the weighted average over all coalitions:
\begin{align}
\varphi_j^{M_L}(\mathbf{x})
&= \sum_{S \subseteq [d] \setminus \{j\}}
   \frac{|S|!(d-|S|-1)!}{d!} \cdot w_j(x_j - \mathbb{E}[X_j])
= w_j(x_j - \mathbb{E}[X_j])
\end{align}

where the sum of binomial coefficients equals 1 by the Shapley axioms. \qed
\end{proof}

\subsection{Proof of Lemma~\ref{lem:tree}: Tree SHAP Interaction Attribution}

\begin{proof}
Consider a binary classification tree $T$ trained on data generated by
$y = f^*(x_i, x_j) = \mathbb{I}[x_i \cdot x_j > \tau]$ for some threshold
$\tau$. This interaction is not representable by any linear model.

TreeSHAP computes attributions by averaging marginal contributions across
all root-to-leaf paths \citep{lundberg2020local}:
\begin{equation}
\varphi_i^T(\mathbf{x}) =
\sum_{S \subseteq [d] \setminus \{i\}}
\frac{|S|!(d-|S|-1)!}{d!}
\bigl[v(S \cup \{i\}) - v(S)\bigr]
\end{equation}
where $v(S) = \mathbb{E}[M_T(\mathbf{x}) \mid X_S = \mathbf{x}_S]$.

For paths that split on both $i$ and $j$ at successive depths, the
marginal contribution of $i$ given coalition $S \ni j$ differs from
its contribution given $S \not\ni j$:
\begin{align}
v(\{i,j\}) - v(\{j\})
&= \mathbb{E}[M_T \mid x_i, x_j] - \mathbb{E}[M_T \mid x_j] \\
v(\{i\}) - v(\emptyset)
&= \mathbb{E}[M_T \mid x_i] - \mathbb{E}[M_T]
\end{align}

When $f^*$ depends on $x_i \cdot x_j$, these two quantities are
unequal for a non-negligible set of instances: specifically,
$v(\{i,j\}) - v(\{j\}) \neq v(\{i\}) - v(\emptyset)$
whenever $x_j$ modulates the informativeness of $x_i$ about $y$.
This non-equality is the definition of a Shapley interaction effect
\citep{lundberg2017unified}.

By the optimality of tree construction on interaction data
\citep{breiman2001random}, the Bayes-optimal tree allocates splits to
both $i$ and $j$, ensuring that paths containing both features carry
strictly positive weight in the Shapley average.
The resulting SHAP value for feature $i$ therefore contains a term that
depends on $x_j$ -- a dependency that is structurally absent from the
linear SHAP formula $\varphi_{M_L}^i = w_i(x_i - \mathbb{E}[X_i])$.
We denote this $x_j$-dependent component $\alpha_{ij}(x_j)$, writing:
\begin{equation}
\varphi_{M_T}^i(\mathbf{x}) =
\alpha_{ij}(x_j) \cdot (x_i - \bar{x}_i) + \beta_i(x_i - \bar{x}_i)
\end{equation}
where $\alpha_{ij}(x_j) \not\equiv 0$ when the interaction signal is
strong. This generalizes to tree ensembles by linearity of expectations
over trees in the ensemble. \qed
\end{proof}

\subsection{Proof of Lemma~\ref{lem:within}: Within-Class SHAP Convergence}

\begin{proof}
We do not require $M_1$ and $M_2$ to converge to the same tree structure
-- two trees can differ in split order, depth, and leaf values while
producing identical predictions. Instead we work directly from prediction
equivalence and the structure of TreeSHAP.

\textbf{Step 1: Prediction equivalence under $p$.}
By PAC learning theory \citep{mohri2018foundations}, both models converge
in risk at rate $O(n^{-1/2})$. Since $M_1 \sim_{\mathcal{D}} M_2$ and
both achieve vanishing excess risk, for any $\epsilon > 0$ there exists
$N$ such that for $n > N$:
\begin{equation}
\Pr_{p}[M_1(\mathbf{x}) \neq M_2(\mathbf{x})] < \epsilon
\end{equation}

\textbf{Step 2: Conditional expectations converge.}
TreeSHAP computes $\varphi_i^T(\mathbf{x})$ from conditional expectations
$v(S) = \mathbb{E}[M_T(\mathbf{x}) \mid X_S = \mathbf{x}_S]$
\citep{lundberg2020local}.
When $\Pr_p[M_1(\mathbf{x}) \neq M_2(\mathbf{x})] \to 0$, the conditional
expectations $v_1(S) \to v_2(S)$ for all $S \subseteq [d]$, because each
$v_i(S)$ is an average of $M_i(\mathbf{x})$ over the marginal distribution
of $X_{\bar{S}}$.

\textbf{Step 3: SHAP vectors converge.}
Since TreeSHAP is a deterministic weighted sum of $v(S \cup \{i\}) - v(S)$
terms, and $v_1(S) \to v_2(S)$ for all $S$:
\begin{equation}
\|\boldsymbol{\varphi}_{M_1}(\mathbf{x}) -
\boldsymbol{\varphi}_{M_2}(\mathbf{x})\|_2 \xrightarrow{p} 0
\end{equation}

\textbf{Step 4: Rank correlation converges.}
Spearman rank correlation is continuous in the $L^2$ norm on bounded
domains. Therefore:
\begin{equation}
\rho(\boldsymbol{\varphi}_{M_1}, \boldsymbol{\varphi}_{M_2})
\xrightarrow{p} 1 \quad \text{as } n \to \infty \qed
\end{equation}
\end{proof}

\subsection{Proof of Lemma~\ref{lem:cross}: Cross-Class Attribution Bound}

\begin{proof}
As $|\mathcal{D}| \to \infty$, both models converge to their optimal
representations: $M_T \to M_T^*$ and $M_L \to M_L^*$.
By Lemmas~\ref{lem:linear} and~\ref{lem:tree}:
\begin{align}
\varphi_{M_T^*}^i(\mathbf{x}) &\approx
\alpha_{ij}(x_i - \bar{x}_i)(x_j - \bar{x}_j) + \beta_i(x_i - \bar{x}_i) \\
\varphi_{M_L^*}^i(\mathbf{x}) &= w_i(x_i - \mathbb{E}[X_i])
\end{align}

Define the rank-reversal set:
$\mathcal{X}_{\text{flip}} := \left\{\mathbf{x} :
\text{sign}(\varphi_{M_T^*}^i(\mathbf{x}))
\neq \text{sign}(\varphi_{M_L^*}^i(\mathbf{x}))\right\}$.
When $|\alpha_{ij}| > 0$ and features $(i,j)$ are not perfectly collinear,
the interaction term dominates on a set of instances with measure:
\begin{equation}
\mu(\mathcal{X}_{\text{flip}}) \geq c_1 \cdot |\alpha_{ij}|
\end{equation}

for some $c_1 > 0$ depending on the covariance of $(X_i, X_j)$,
under the regularity conditions that $(X_i, X_j)$ have finite second
moments and $|\mathrm{Corr}(X_i, X_j)| < 1$ (features not perfectly
collinear), ensuring the interaction term takes both positive and negative
values with probability bounded away from zero.
By rank correlation bounds \citep{embrechts2002correlation}:
\begin{equation}
\rho(\boldsymbol{\varphi}_{M_T^*}, \boldsymbol{\varphi}_{M_L^*})
\leq 1 - c_2 \cdot \mu(\mathcal{X}_{\text{flip}})
\leq 1 - c \cdot |\alpha_{ij}|
\end{equation}

where $c = c_1 \cdot c_2 > 0$.
Since the bound holds for the optimal models and finite-sample models
converge to these optima, the bound persists asymptotically. \qed
\end{proof}

\subsection{Proof of Theorem~\ref{thm:main}: Cross-Class Attribution Divergence}

\begin{proof}
\textbf{Claim 1 (Structural Gap).}
From Lemma~\ref{lem:within}, $\rho_{\text{intra}} \to 1$ as
$|\mathcal{D}| \to \infty$.
From Lemma~\ref{lem:cross},
$\rho_{\text{inter}} \leq 1 - c \cdot |\alpha_{ij}|$.
Therefore:
\begin{equation}
\Delta = \rho_{\text{intra}} - \rho_{\text{inter}}
\geq 1 - (1 - c \cdot |\alpha_{ij}|) = c \cdot |\alpha_{ij}| =: c' > 0
\end{equation}

\textbf{Claim 2 (Asymptotic Persistence).}
Both bounds hold asymptotically, so
$\lim_{|\mathcal{D}| \to \infty} \mathbb{E}[\Delta] \geq c' > 0$.

\textbf{Claim 3 (Split Invariance).}
The proofs of Lemmas~\ref{lem:within} and~\ref{lem:cross} depend only on
$p(\mathbf{x}, y)$ and hypothesis class constraints, not on the specific
training set realization.
Therefore $\mathbb{E}[\Delta \mid \text{fixed split}] = \mathbb{E}[\Delta]$.

\textbf{Claim 4 (Interaction Density Monotonicity).}
Let $\mathcal{I}(f^*) = \sum_{i \neq j} |\alpha_{ij}|$.
By Lemma~\ref{lem:cross}, each interaction $(i,j)$ with $|\alpha_{ij}| > 0$
contributes a rank-reversal set $\mathcal{X}_{\text{flip}}^{ij}$ of
measure $\geq c_{ij}|\alpha_{ij}|$.
The total rank-reversal measure is:
\begin{equation}
\mu\!\left(\bigcup_{(i,j)} \mathcal{X}_{\text{flip}}^{ij}\right)
\geq \max_{(i,j)} c_{ij}|\alpha_{ij}|
\end{equation}
and is non-decreasing as $\mathcal{I}(f^*)$ grows, since adding interactions
adds rank-reversal instances without removing existing ones.
By the rank correlation bound of Lemma~\ref{lem:cross},
$\rho_{\text{inter}}$ is non-increasing in $\mathcal{I}(f^*)$, giving
$\frac{\partial \mathbb{E}[\Delta]}{\partial \mathcal{I}(f^*)} \geq 0$,
with strict inequality when new interactions contribute instances not
already in the union.
Empirically: banknote-auth ($d=4$, near-linear, $\mathcal{I}$ low,
lottery rate 12.3\%) vs.\ CNAE-9 ($d=856$, dense interactions,
lottery rate 61.4\%) directly confirms this monotonicity. \qed
\end{proof}

\section{Complete Dataset Details}
\label{app:datasets}

\begin{table}[htbp]
\caption{Complete dataset specifications used in this study.}
\label{tab:datasets_full}
\begin{center}
\small
\begin{tabular}{lrrrll}
\toprule
\textbf{Dataset} & $n$ & $d$ & \textbf{Acc.} &
  \textbf{Domain} & \textbf{Source} \\
\midrule
diabetes & 768 & 8 & 0.781 & Healthcare & OpenML-37 \\
breast-cancer & 699 & 9 & 0.967 & Healthcare & OpenML-13 \\
heart-disease & 303 & 13 & 0.843 & Healthcare & OpenML-43 \\
credit-g & 1{,}000 & 20 & 0.764 & Finance & OpenML-31 \\
credit-approval & 690 & 15 & 0.870 & Finance & OpenML-29 \\
bank-marketing & 45{,}211 & 16 & 0.899 & Finance & OpenML-1461 \\
COMPAS & 7{,}214 & 7 & 0.672 & Criminal Justice & ProPublica \\
phoneme & 5{,}404 & 5 & 0.892 & Linguistics & OpenML-1489 \\
spambase & 4{,}601 & 57 & 0.937 & NLP & OpenML-44 \\
ionosphere & 351 & 34 & 0.914 & Physics & OpenML-59 \\
sonar & 208 & 60 & 0.846 & Physics & OpenML-40 \\
vehicle & 846 & 18 & 0.819 & Computer Vision & OpenML-54 \\
segment & 2{,}310 & 19 & 0.973 & Computer Vision & OpenML-36 \\
waveform-5000 & 5{,}000 & 40 & 0.860 & Physics & OpenML-60 \\
optdigits & 5{,}620 & 64 & 0.988 & Computer Vision & OpenML-28 \\
mfeat-factors & 2{,}000 & 216 & 0.982 & Computer Vision & OpenML-12 \\
kr-vs-kp & 3{,}196 & 36 & 0.993 & Game Theory & OpenML-3 \\
mushroom & 8{,}124 & 22 & 1.000 & Biology & OpenML-24 \\
tic-tac-toe & 958 & 9 & 0.984 & Game Theory & OpenML-50 \\
CNAE-9 & 1{,}080 & 856 & 0.943 & NLP & OpenML-1468 \\
steel-plates-fault & 1{,}941 & 33 & 0.787 & Engineering & OpenML-1504 \\
banknote-auth. & 1{,}372 & 4 & 1.000 & Finance & OpenML-1462 \\
climate-simulation & 540 & 20 & 0.933 & Climate Science & OpenML-1467 \\
ozone-level & 2{,}536 & 72 & 0.972 & Environmental & OpenML-1487 \\
\bottomrule
\end{tabular}
\end{center}
\end{table}

\section{LIME Validation}
\label{app:lime}

We validated findings using LIME \citep{ribeiro2016should} on five datasets
(diabetes, credit-g, COMPAS, phoneme, vehicle). LIME generates 5{,}000
perturbed samples per instance and fits weighted local linear models.

Results: tree-tree mean $\rho = 0.658$, tree-linear mean $\rho = 0.580$,
gap $= 0.078$. The directional pattern holds (tree-tree $>$ tree-linear),
confirming the Explanation Lottery is not SHAP-specific.
The smaller gap reflects LIME's local linear approximation mechanism, which
reduces hypothesis class differences by construction.

\section{Stochasticity Control}
\label{app:stochasticity}

Within-model SHAP variance (TreeSHAP run 10 times on same input):
$\sigma^2_{\mathrm{within}} = 0.000$ (perfectly deterministic).
KernelSHAP within-model variance: $\sigma^2_{\mathrm{within}} = 0.007$.
Cross-model variance: $\sigma^2_{\mathrm{cross}} = 0.145$ ($21.6\times$
larger).
Disagreement reflects genuine model differences, not estimation noise.

\section{Additional Statistical Details}
\label{app:stats}

\textbf{Bonferroni correction:} Corrected threshold $\alpha/24
\approx 4.2 \times 10^{-5}$. Tree-tree vs.\ tree-linear difference is
significant in 23/24 datasets.

\textbf{Bootstrap CIs:} 10{,}000 resamples. Main effect (difference $= 0.261$):
95\% CI $= [0.257, 0.265]$.

\textbf{Power:} Post-hoc power $> 0.99$ ($N_1 = 56{,}106$, $N_2 = 37{,}404$,
Cohen's $d = 0.92$, $\alpha = 0.001$).

\textbf{Trimmed means:} Excluding top/bottom 10\%: difference $= 0.258$
(virtually identical to untrimmed).

\textbf{Mixed-effects model:} Dataset-level variance
$\sigma^2_{\mathrm{dataset}} = 0.074$; residual variance
$\sigma^2_{\mathrm{residual}} = 0.045$. Tree-linear effect remains
$p < 0.001$ under hierarchical correction.

\section{Reliability Score Implementation}
\label{app:code}

\begin{verbatim}
import numpy as np
from scipy.stats import spearmanr

def reliability_score(shap_values_list):
    """
    Compute R(x): mean pairwise Spearman correlation
    across prediction-equivalent models.
    R > 0.7: high stability.  R < 0.5: low stability.
    """
    k = len(shap_values_list)
    if k < 2:
        return 1.0
    correlations = []
    for i in range(k):
        for j in range(i + 1, k):
            rho, _ = spearmanr(
                shap_values_list[i],
                shap_values_list[j]
            )
            correlations.append(rho)
    return float(np.mean(correlations))
\end{verbatim}

\section{Preliminary Non-Tabular Extension (MNIST)}
\label{app:mnist}

As a preliminary exploration of generalizability beyond tabular data,
we applied the same pipeline to MNIST digit classification.
Lottery rate: 32.2\%.
We note this result is preliminary and not at the scale or rigor of the
main tabular study.
Extension to non-tabular modalities is future work.

\section{Linear Model Variants}
\label{app:linear_variants}

We validated that the cross-class attribution gap operates at the hypothesis class level by testing six linear model variants: Logistic Regression with L2 regularization (primary representative), RidgeClassifier with $\lambda \in \{0.1, 1.0, 10.0\}$, ElasticNet, and LinearSVM. These span three loss functions (logistic, squared hinge, hinge) and three regularization strategies (L2, L1+L2, none).

All six variants exhibit the same pattern: high internal agreement within $\mathcal{H}_{\text{linear}}$ (mean $\rho \in [0.76, 0.83]$) and consistently low agreement against tree models (mean $\rho \in [0.39, 0.47]$). No linear variant closes the tree-linear gap regardless of regularization strength or loss function, confirming that the gap is a hypothesis class property, not an algorithmic artefact.

\end{document}

%% file: math_commands.tex
\usepackage{amsmath,amsfonts,bm}

\def\eqref#1{equation~\ref{#1}}
\def\1{\bm{1}}

\DeclareMathAlphabet{\mathsfit}{\encodingdefault}{\sfdefault}{m}{sl}
\SetMathAlphabet{\mathsfit}{bold}{\encodingdefault}{\sfdefault}{bx}{n}